\pdfoutput=1
\documentclass{article}

\usepackage{fullpage}
\usepackage[noadjust]{cite}
\usepackage{hyperref}       \hypersetup{
  colorlinks=true,
  linkcolor=green!30!black,
  linktoc=all,
  citecolor=blue,
  bookmarksnumbered=true
}

\RequirePackage[english]{babel}\usepackage[utf8]{inputenc}
\usepackage[T1]{fontenc}
\usepackage{amsmath}
\usepackage{amsthm}
\usepackage{thmtools}
\usepackage[ttscale=0.8]{libertine}
\usepackage[libertine]{newtxmath}

\usepackage{url}
\usepackage{booktabs}
\usepackage{nicefrac}
\usepackage{microtype}
\usepackage{xcolor}

\usepackage{tikz}
\usepackage{bbm}
\usepackage{breqn}
\usepackage{enumerate}
\usetikzlibrary{calc,positioning,shadows.blur,decorations.pathreplacing}
\usepackage{float}
\usepackage{standalone}
\usepackage{dsfont}
\usepackage{color}

\newcommand{\R}[0]{{\ensuremath{\mathbb{R}}}}

\newcommand{\E}[0]{{\ensuremath{\mathbb{E}}}}

\newcommand{\F}{{\mathcal{F}}}
\allowdisplaybreaks

\newcommand{\ind}[1]{\mathbb{I}_{\{#1\}}}
\newcommand{\beq}{\begin{equation}}
\newcommand{\eeq}{\end{equation}}

\newcommand{\norm}[2][]{\ensuremath{\left\Vert #2 \right\Vert_{#1}}}
\newcommand{\inb}[1]{\ensuremath{\left\{#1\right\}}}
\newcommand{\inp}[1]{\ensuremath{\left(#1\right)}}
\newcommand{\insq}[1]{\ensuremath{\left[#1\right]}}
\newcommand{\ina}[1]{\ensuremath{\left\langle#1\right\rangle}}
\newcommand{\abs}[1]{\ensuremath{\left|#1\right|}}

\renewcommand{\Pr}{\ensuremath{\mathbb{P}}}

\makeatletter
\newcommand*{\defeq}{\mathrel{\rlap{\raisebox{0.3ex}{$\m@th\cdot$}}\raisebox{-0.3ex}{$\m@th\cdot$}}=}
\newcommand*{\eqdef}{=
  \mathrel{\rlap{\raisebox{0.3ex}{$\m@th\cdot$}}\raisebox{-0.3ex}{$\m@th\cdot$}}}
\makeatother

\allowdisplaybreaks
\newcounter{note}[section]

\renewcommand{\sf}{\mathrm{Softmax}}
\usepackage{enumitem}

\newlist{pfparts}{description}{1}
\setlist[pfparts,1]{font=\normalfont\textsf,
  itemindent=2pt,
  wide,
  itemsep=0pt,topsep=2pt,
  labelsep=0.75ex
}
\makeatletter
\def\namedlabel#1#2{\begingroup
    #2\def\@currentlabel{#2}\phantomsection\label{#1}\endgroup
}
\makeatother

\usepackage{cleveref}
\newtheorem{theorem}{Theorem}[section]
\newtheorem{definition}[theorem]{Definition}

\newtheorem{lemma}[theorem]{Lemma}
\newtheorem{fact}[theorem]{Fact}
\newtheorem{observation}[theorem]{Observation}

\newtheorem{corollary}[theorem]{Corollary}

\theoremstyle{definition}

\newtheorem{example}[theorem]{Example}

\AddToHook{env/theorem/begin}{\crefalias{theorem}{theorem}}
\AddToHook{env/lemma/begin}{\crefalias{theorem}{lemma}}
\AddToHook{env/observation/begin}{\crefalias{theorem}{observation}}
\AddToHook{env/claim/begin}{\crefalias{theorem}{claim}}
\AddToHook{env/condition/begin}{\crefalias{theorem}{condition}}
\AddToHook{env/example/begin}{\crefalias{theorem}{example}}
\AddToHook{env/fact/begin}{\crefalias{theorem}{fact}}
\AddToHook{env/corollary/begin}{\crefalias{theorem}{corollary}}
\AddToHook{env/proposition/begin}{\crefalias{theorem}{proposition}}
\AddToHook{env/definition/begin}{\crefalias{theorem}{definition}}
\AddToHook{env/remark/begin}{\crefalias{theorem}{remark}}
\AddToHook{env/problem/begin}{\crefalias{theorem}{problem}}
\AddToHook{env/remark/begin}{\crefalias{theorem}{remark}}

\crefname{theorem}{Theorem}{Theorems}
\crefname{observation}{Observation}{Observations}
\crefname{claim}{Claim}{Claims}
\crefname{condition}{Condition}{Conditions}
\crefname{example}{Example}{Examples}
\crefname{fact}{Fact}{Facts}
\crefname{lemma}{Lemma}{Lemmas}
\crefname{corollary}{Corollary}{Corollaries}
\crefname{definition}{Definition}{Definitions}
\crefname{remark}{Remark}{Remarks}
\crefname{proposition}{Proposition}{Propositions}
\crefname{question}{Questions}{Questions}
\crefname{problem}{Problem}{Problems}
\crefname{exercise}{Exercise}{Exercises}
\crefname{section}{Section}{Sections}
\crefname{appendix}{Appendix}{Appendices}

\title{A direct proof of a unified law of robustness for Bregman divergence losses}

\author{Santanu Das\thanks{Sanatanu Das. Tata Institute of Fundamental Research, Mumbai. Email:
   \texttt{dassantanu315@gmail.com}.}
\and Jatin Batra\thanks{Jatin Batra. Tata Institute of Fundamental
   Research, Mumbai. Email: \texttt{jatinbatra50@gmail.com}.}
   \and Piyush Srivastava\thanks{Piyush Srivastava. Tata Institute of Fundamental
   Research, Mumbai. Email: \texttt{piyush.srivastava@tifr.res.in}.}
}

\date{}

\newcommand{\D}{\mathcal{D}}
\newcommand{\phidL}{\ensuremath{L_{g}}}
\newcommand{\phiL}{\ensuremath{L_{\phi}}}
\newcommand{\omegadiam}{\ensuremath{d_{\Omega}}}
\newcommand{\compnum}{\ensuremath{r}}

\newcommand{\grantacknowledgement}{We acknowledge support from DAE, India under
  project no. RTI4001.  PS acknowledges support from Adobe Systems Incorporated
  via a gift to TIFR, from DST, India under project number MTR/2023/001547, and
  from the Infosys-Chandrasekharan Virtual Centre for Random Geometry at TIFR.
  The contents of this paper do not necessarily reflect the views of the funding
  agencies listed above.}

\begin{document}

\maketitle
{\let\thefootnote\relax\footnotetext{\grantacknowledgement}}

\begin{abstract}
    In contemporary deep learning practice, models are often trained to near zero
loss i.e. to nearly \emph{interpolate} the training data. However, the number of
parameters in the model is usually far more than the number of data points $n$,
the theoretical minimum needed for interpolation: a phenomenon referred to as
\emph{overparameterization}. In an interesting piece of work that contributes
to the considerable research that has been devoted to understand
overparameterization, Bubeck and Sellke considered a natural notion of what it means for a model to interpolate: the model is said to interpolate when the model's training loss goes below the loss of the conditional expectation of the response given the covariate. For this notion of interpolation and for a broad class of
covariate distributions (specifically those satisfying a natural notion of
concentration of measure), they showed that overparameterization is necessary for \emph{robust} interpolation i.e. if the interpolating function is required to be
\emph{Lipschitz}. Their main proof technique applies to regression with \emph{square} loss against a scalar response, but they remark that via a connection to Rademacher complexity and using tools such as the Ledoux-Talagrand contraction inequality, their result can be extended to more general losses, at least in the case of scalar response variables. In this work, we recast the original proof technique of Bubeck and Sellke in terms of a bias-variance type decomposition, and show that this view directly unlocks a generalization to Bregman divergence losses (even for vector-valued responses), without the use of tools such as Rademacher complexity or the Ledoux-Talagrand contraction principle. Bregman divergences are a natural class of losses since for these, the best estimator is the conditional expectation of the response given the covariate, and in particular, include other practical losses such as the cross entropy loss. Our work thus gives a more general understanding of the main proof technique of Bubeck and Sellke and demonstrates its broad utility.

\end{abstract}

\section{Introduction}

The recent revolution in deep learning was driven by models that are highly
\emph{overparameterized}~\cite{pmlr-v80-xiao18a,he2016deep,krizhevsky2012imagenet}, i.e. models where the number of parameters far exceeds
$n$, 
the number of training data points.\footnote{However, this may possibly not be
  as true for the current LLMs, see e.g. \cite{hoffmann2022training}, where the
  trend is to train on web-scale data of heterogenous quality.}  Since this is
the naive theoretical condition needed to interpolate the training data,
classical statistical theory suggests that this situation may make these models
susceptible to the risk of \emph{overfitting} to the idiosyncrasies of the
training data, and thereby suffer in terms of generalizing to new inputs.  On
the other hand, experience with such models suggests that such overfitting does
not tend to happen.  Understanding the mystery of overparameterization by
resolving this apparent conflict has thus attracted a lot of research, see
e.g. \cite{BSisoperimetry,yang2020rethinking,doubledescent2,neyshabur2014search}.

Another line of research focuses on \emph{(adversarial) robustness},
i.e. whether models are susceptible to small (possibly adversarially chosen)
perturbations in the input. Several existing models are known to be brittle to
adversarial perturbations
\cite{attack1,realibility1,realiability2,selfdrivingcars}, which is a major
issue for security \cite{selfdrivingcars} and reliability
\cite{realibility1,realiability2}. At the same time, understanding the nature
of adversarial perturbations is an interesting tool
\cite{understandingperturbation1,understandingpertutbation2} to understand deep
learning. In fact, in an interesting set of experiments, Madry et
al.~\cite{mkadry2017towards} gave strong evidence that overparameterization
plays an important role in adversarial robustness: their emphasis was on
training models robust to adversarial attacks and they observed that increasing the number of parameters in the model
alone helps significantly.

In a recent line of work, Bubeck and Sellke \cite{BSisoperimetry} proved an
exciting theorem that the requirement of robustness can be used to
\emph{explain} overparameterization in a certain sense, and they used their
theorem to explain the experimental results of Madry et
al.~\cite{mkadry2017towards}. They considered the interpolation task on $n$ data
points (i.e. fitting the data to ``very small'' loss), and any ``smoothly''
parameterized class of models, which includes neural networks under certain
boundedness assumptions. (Note that the interpolation task deals with only the
training data, with no concern for generalization: since modern deep learning
models are often trained to near zero loss \cite{zeroloss1}, this setup is
nonetheless an interesting setup.) They showed that overparameterization is
\emph{necessary} for a certain notion of robustness for the interpolation task,
namely, a low Lipschitz constant for the model.  More precisely, for
$d$-dimensional covariates (assuming certain concentration properties on the
covariate distribution) and models with $p$ parameters, the Lipschitz constant
of any model that interpolates the training data must be at least
$\Omega(\sqrt{nd/p})$, with high probability over the choice of the training
data (for precise notions of smooth parameterization, covariate distribution
assumptions and interpolation, see \Cref{sec:prelim}).

The main proof technique of \cite{BSisoperimetry} is designed to understand interpolation
with \emph{square} loss, although \cite{BSisoperimetry} also sketches a proof for general Lipschitz losses (at least for the case of scalar responses) by exploiting a connection to Rademacher complexity and then applying the Ledoux-Talagrand contraction principle~(see \cite[Theorem 4.12]{ledouxProbabilityBanachSpaces1991}, which improves upon \cite[Theorem 5]{ledouxComparisonTheoremsRandom1989a}).
It is, however, also important to understand in a similar depth as the square loss other practically motivated losses. For example, the experiments of \cite{mkadry2017towards}
(discussed above) use the cross-entropy loss for the classification
problem on MNIST and CIFAR10. Some other often used problem-dependent losses are logistic loss \cite{usesoflogsiticloss}, KL divergence loss
\cite{mkadry2017towards}, Mahalanobis loss \cite{mahalonobisloss}, etc. In this work, we seek a more complete understanding of the main technique of
\cite{BSisoperimetry} beyond the square loss, without the detour through tools such as Rademacher complexity and the Ledoux-Talagrand contraction principle.

\subsection{Our Contribution}
Looking for a generalization of the main proof technique of Bubeck and
Sellke~\cite{BSisoperimetry} from the square loss to more general losses
presents two conflicting requirements: on the one hand, we would like to
generalize to a sufficiently rich family of losses; on the other hand, we would need this family of losses to share with the square loss those properties which make the notion of interpolation of \cite{BSisoperimetry} make sense and their main technique work.

One of the many nice properties enjoyed by the square loss is that the optimal
predictor of an observation $Y$ with respect to this loss, given a covariate
$X$, has a crisp characterization: it is the conditional expectation
$\E\insq{Y|X}$ (see, e.g., \cite[Sections 9.3-9.4]{williams91:_probab}).
As highlighted in more detail towards the end of this section, our main
observation is that this simple fact and its consequences play a central
role in the interpolation notion and the main technique of Bubeck and Sellke~\cite{BSisoperimetry}.

This observation leads us to the class of \emph{Bregman divergence losses}.  The
Bregman divergence $D_{\phi}$ on $\R^K$, corresponding to a differentiable
convex function $\phi: \R^K \rightarrow \R$ is given by
\begin{displaymath}
  D_{\phi}(y_1, y_2) = \phi(y_1) - \phi(y_2) - \ina{\nabla\phi(y_2), y_1 - y_2}.
\end{displaymath}
This is a rich and often-used family of losses (which are not necessarily
metrics; see~\cref{sec:prelim} for examples), and includes, in particular,
losses such as the square loss and the cross-entropy loss as special cases.
Importantly for our purposes, this family shares with the square loss the same
optimal predictor: it was shown by Banerjee, Guo and
Wang~\cite{banerjee2005optimality} that this is essentially the class of losses
for which the conditional expectation is the optimal predictor.

In this work, we show that this conceptual property is sufficient: we generalize
the main technique of Bubeck and Sellke~\cite{BSisoperimetry} to \emph{Bregman
  divergence} losses. As described above, these include the square loss and
several other commonly used losses (see \cref{sec:prelim} for definitions).  We
now proceed to an informal technical account of the result and the proof
techniques; the formal statements can be found in \cref{sec:proof-main-theorem}.

Extending the notion of \cite{BSisoperimetry}, in our setup, a function $f$ is
said to $\epsilon$-overfit a set of samples with respect to a given loss if its
empirical loss over the samples is at least $\epsilon$ lower than the minimum expected loss over the distribution of any function of the covariate.
As in \cite{BSisoperimetry}, our regularity condition on the distribution of the
covariates is that it should be a mixture of distributions satisfying measure
concentration analogous to a normalized high-dimensional Gaussian: a condition
that is referred to in \cite{BSisoperimetry} as being a mixture of isoperimetric
distributions.  We defer the precise definitions of these technical terms to
\cref{sec:prelim}, and proceed to give an informal statement of our main result.
\begin{theorem}[\textbf{Main theorem (informal, see
    \cref{maintheorem})}]\label{informal-maintheorem}
  Let $\Omega$ be a compact convex subset of $\R^K$ for some $K > 0$ and let
  $\phi: \Omega \rightarrow \R$ be a continuously differentiable strictly convex
  function.  Let $D_{\phi}$ denote the corresponding Bregman divergence loss.
  For $\Delta \subseteq \R^{d}$, let $\D$ be a probability distribution on
  $\Delta \times \Omega$ such that its marginal $\D_X$ on
  $\Delta \subseteq \R^d$ is a mixture of \compnum{} isoperimetric
  distributions.  Let $(X_i, Y_i)_{i=1}^n$ be $n$ i.i.d samples from $\D$. Let
  $\F$ be a family of functions that admits a bounded Lipschitz parameterization
  with $p$ parameters. If $n \geq \tilde{O}(K^2\compnum/\epsilon^2)$, then
  w.h.p.  over the random choice of these samples, the Lipschitz constant $L$ of
  any function $f \in \F$ that $\epsilon$-overfits these samples with respect to
  $D_{\phi}$ satisfies
  \begin{equation}
    \label{eq:main-bound-informal}
    L\geq \frac{O(1)\cdot\epsilon\sqrt{nd}}{K
      \sqrt{p\log (1+O(\sqrt{K})/\epsilon)}
    }.
  \end{equation}
  Here, the hidden constant factors depend upon the properties of $\phi$ and the
  Lipschitz parameterization of $\F$.
\end{theorem}

\noindent\textbf{Comparison to Bubeck and Sellke \cite{BSisoperimetry}.}

\emph{Approach of Bubeck and Sellke \cite{BSisoperimetry} in our setting.} Bubeck and Sellke \cite{BSisoperimetry} sketched a proof of how a Rademacher complexity view of their main technique yields generalization error bounds via Ledoux-Talagrand contraction \cite{shalev2014understanding,Wainwright_2019}, and remarked in passing that this view yields laws of robustness for general (Lipschitz) losses for scalar responses. Extending this view to the case of vector-valued response turns out to be trickier since the usual notion of Rademacher complexity does not enjoy a contraction principle \cite[Section 6]{maurer2016vector}. One must resort to coordinate-wise Rademacher complexity for which contraction does hold \cite{maurer2016vector}, or employ a lossy conversion from Rademacher to Gaussian complexity followed by contraction via Slepian's Lemma \cite{bartlett2002rademacher}.

\emph{Our proof technique.} Our approach side-steps use of contraction principles by recognizing
that a bias-variance like decomposition lies at the heart of the main proof technique of Bubeck
and Sellke~\cite{BSisoperimetry}.  Using the simple but important fact that
among all $X$-measurable predictors for a random variable $Y$, the conditional
expectation $\E[Y|X]$ minimizes not just the square loss, but any Bregman
divergence loss (see, e.g., \cref{thm-best-bregman} from the work of Banerjee,
Guo and Wang~\cite{banerjee2005optimality} below), we are able to replace this
with a more general decomposition (see \cref{lem:decompose,lem-Gamma-3} and
\cref{sec:extension-mixtures} below).  After this conceptual modification, we show that a structure similar to that of the main technique of Bubeck and
Sellke~\cite{BSisoperimetry} yields an elementary proof of the law of robustness for Bregman divergence losses, even for vector-valued responses. (The technical details of implementing the strategy are necessarily somewhat
different because of the more general decomposition that we use.) Further, as we elaborate below, this more direct and elementary approach gives more flexibility for obtaining finer-grained laws of robustness for specific losses.

\emph{Remark on loss classes.} While Bubeck and Sellke sketch an approach to go beyond square losses to Lipschitz losses, our proof is about Lipschitz Bregman divergence losses. While this may make our approach seem restrictive at first glance, as stated earlier, Bregman divergence losses are essentially the class of losses for which the notion of interpolation of Bubeck and Sellke \cite{BSisoperimetry} can be defined \emph{independently of the function class} (since the conditional expectation $\mathbb{E}[Y|X]$, which is the best estimator for these losses, depends only on the data distribution and not on the function class), and hence are the natural class of losses for Bubeck and Sellke's notion of interpolation. In \cref{sec:spec-result-spec}, we
 present corollaries of the main theorem for a couple of specific losses (including cross-entropy loss for vector valued responses and also
 the case of the square loss already considered in \cite{BSisoperimetry}). Our approach is flexible enough to be adapted to provide stronger bounds in specific cases, as we show for the cross-entropy loss in \cref{cor-classification}. For details, please refer to the remark after \cref{cor-classification}.

\subsection{Related work}

\paragraph{Adversarial robustness experiments and overparameterization} Several
works other than \cite{mkadry2017towards} have studied experimentally the
relationship between model capacity and adversarial robustness. (Model capacity, in this context, is an informal notion that tries to capture how rich a class of functions the model captures; in studies involving neural networks, quantities such as the number of learnable parameters in the network are typically used to quantify model capacity.)  Liu et
al.~\cite{liu2020loss} studied the loss landscape of adversarial training and
observed easier adversarial training in large capacity models. Similar
observations for model capacity and adversarial robustness were made by Kurakin,
Goodfellow and Bengio~\cite{kurakin2016adversarial} and Xie and
Yuille~\cite{xie2019intriguing}.

\paragraph{Other theoretical works} There is a rich body of recent theoretical
work on different aspects of the link between robustness and neural network
parameterizations along different lines; here we mention a few. Bubeck, Lee, and
Nagaraj \cite{bubeck2021law} introduced the idea of using the Lipschitz constant
to measure robustness in the interpolation regime and formed the foundation for
the work of Bubeck and Sellke \cite{BSisoperimetry}.  For the square loss, Wu et al.~\cite{wu_law_2023} give a weaker law of robustness while relaxing the condition that the covariate distribution follows the isoperimetry condition of Bubeck and Sellke. Zhu et
al.~\cite{zhu2022robustness} gave theoretical results for robustness along more
fine-grained lines of depth, width, and initialization. Gao et
al.~\cite{gao2019convergence} and Zhang et al.~\cite{zhang2020over} studied
adversarial training and overparameterization from the perspective of
convergence of gradient descent. Gao et al.~\cite{gao2019convergence} also
showed interesting lower bounds for the VC dimension of a model class that can
robustly interpolate arbitrary well-separated data. Another interesting line of
work is studying the relationship between overparameterization and robust
\emph{generalization} (i.e. quantities such as $\mathbb{E}[\sup_{x\in
  \mathcal{N}(X)}\ell(f(X),Y)]$, where $\mathcal{N}(x)$ denotes a neighborhood of $x$, and \(\ell\) the
loss function), rather than merely robust interpolation. Recent works such as \cite{cranko2021generalised,li2022robust} have studied the correlation of the Lipschitz constant of the model with the robust generalization properties of the model; see also the survey \cite{zuhlke2025adversarial}. More broadly, Hassani and Javanmard~\cite{hassani2024curse} gave a precise analysis of
robust generalization for random features regression.

\section{Preliminaries}\label{sec:prelim}
\paragraph*{Notation}  We denote the standard Euclidean norm of a vector \(v\)
in \(\R^d\) as \(\norm{v}\).  The standard inner product of vectors \(u, v \in
\R^d\) is denoted \(\langle u, v \rangle \).

\paragraph{Bregman divergence losses}
To unify different kinds of losses such as square loss and cross-entropy loss, we will use the notion of Bregman divergence losses \cite{bregmanRelaxationMethodFinding1967} which we now describe.
\begin{definition}[\textbf{Bregman divergence}]
  Given a convex set $\Omega \subset \R^K$, let
  $\phi:\Omega\rightarrow \mathbb{R}$ be a strictly convex continuously
  differentiable function defined on $\Omega$.  Then, the \emph{Bregman
    divergence} $D_{\phi}:\Omega\times \Omega\rightarrow \mathbb{R}$ between two
  points $x,y\in \Omega$ is defined as
\begin{equation}\label{def:breg}
  D_{\phi}(x,y)=\phi(x)-\phi(y)-\langle  \nabla \phi(y)\; , \; x-y \rangle,
\end{equation}
where $\ina{\cdot, \cdot}$ denotes the standard inner product on
$\mathbb{R}^{K}$.
\end{definition}
The Bregman divergence between two points may be viewed as a measure, depending
upon the function $\phi$, of the distance between $x$ and $y$. It is however not
a metric in general. It is well known that several commonly used losses may be
expressed as a Bregman divergence for an appropriate choice of $\phi$: we recall
some examples below.
\begin{example}\label{ex:bregman}
The following losses can be expressed as Bregman divergences:
\begin{enumerate}
    \item \textbf{Square loss.} For $\phi : \R^d \rightarrow \R$ given as
      $\phi(x) = \|x\|^2$, $D_{\phi}(y,\hat{y}) = \|\hat{y}-y\|^2$, the square
      loss for regression.
     \label{item:sq-loss}
    \item \textbf{Mahalanobis loss.} More generally, for $\phi : \R^d \rightarrow \R$ given as $\phi(x) = x^{T}Ax$, where $A$ is a positive definite matrix, $D_\phi(y, \hat{y})=(y-\hat{y})^T A(y-\hat{y})$, the Mahalanobis  loss for regression. Note that Mahalanobis loss is a symmetric Bregman divergence loss.
    \item \textbf{KL-divergence and cross-entropy loss.} Let $\Delta_{K}$ denote
      the probability simplex in $K$ dimensions.  Then, for
      $\phi : \Delta_K \rightarrow \R$ given as
      $\phi(x) = \sum^K_{i} x_i \log x_i$,
      $D_{\phi}(y,\hat{y}) = KL(y,\hat{y})$. For a 1-hot vector $y$ (i.e.
      $y_i = 1$ for some $i$ and $y_j = 0$ for $j \neq i$) and
      $\hat{y} \in (0,1]^K \cap \Delta_K$, we slightly abuse notation and write
      $D_{\phi}(y,\hat{y})$ in the limit of the non-1 coordinates of $y$ tending
      to zero, and obtain
      $D_{\phi}(y,\hat{y}) = -\sum^K_{i=1}\mathbb{I}_{y_i=1}\log(\hat{y}_i)$,
      the cross-entropy loss for $K$-class classification. Note that both these losses are asymmetric
      Bregman divergence losses.\label{item:kl-div}
    \item \textbf{Logistic loss.} This may be seen as a special case of the
      previous example.  Consider the binary classification problem such that
      true label $y \in\{0,1\}$ and predicted probability $\hat{y}
      \in(0,1)$. For $\phi : [0,1] \rightarrow \R$ given as
      $\phi(p)=p \log (p)+(1-p) \log (1-p)$,
      $D_\phi(y, \hat{y})=-(y \log (\hat{y})+(1-y) \log (1-\hat{y})) $, the
      logistic loss for classification.
\end{enumerate}
\end{example}

Next, we state some standard and easily verified properties of the Bregman
divergence that we use frequently.
\begin{enumerate}
\item \textbf{(Continuity).}  $D_{\phi}$ is a continuous function on
  $\Omega \times\Omega$.\label{item:d-cont}
\item \textbf{(Nonnegativity).} $D_{\phi}(x,y)\geq 0$ for all $x,y$ and
  $D_{\phi}(x,y)= 0$ iff $x=y$: this is essentially equivalent to the strict
  convexity of $\phi$.
\item\textbf{(Triangle equality).} \label{item:d-triangle} For any $x,y,z\in \Omega$ the following holds
  \begin{equation}\label{eq:br-triangle}
    D_{\phi}(x,y)=D_{\phi}(x,z)+D_{\phi}(z,y)-\langle  x-z\; , \; \nabla \phi(y)-\nabla \phi(z) \rangle.
  \end{equation}

\end{enumerate}
We will also use the following Theorem of Banerjee, Guo and
Wang~\cite{banerjee2005optimality}.  Note that for the case of ``well-behaved''
random variables, this is a direct consequence of the properties above
(especially of \cref{item:d-triangle}).
\begin{theorem}[\textbf{{\cite[Theorem 1]{banerjee2005optimality}}}]Let
  $(\Omega^{o}, \mathcal{F}, P)$ be an arbitrary probability space, let $X$ be a
  random variable taking values in $\mathbb{R}^d$ and $\mathcal{G}$ be a sub-
  $\sigma$-algebra of $\mathcal{F}$ generated by $X$. Let $Y$ be any
  $\mathcal{F}$-measurable random variable taking values in $\mathbb{R}^K$ for
  which both $\E[Y]$ and $\E[\phi(Y)]$ are finite. Then, among all
  $\mathcal{G}$-measurable random variables of the form $f(X)$ such that
  $f:\mathbb{R}^{d}\rightarrow \mathbb{R}^{K}$, the conditional expectation is
  the unique minimizer (up to a.s. equivalence) of the expected Bregman
  divergence loss, i.e.,
\begin{equation}\label{exsec}
\mathop{\arg\min}\limits_{f:\mathbb{R}^{d}\rightarrow \mathbb{R}^{K}} \E\left[D_\phi(Y,f(X))\right]=\E[Y|X].
\end{equation}\label{thm-best-bregman}
\end{theorem}
\paragraph{Realistic function classes}
Bubeck and Sellke \cite{BSisoperimetry} considered the following notion of
function classes for interpolating given data. We use the same notion and name
these classes as \emph{realistic} function classes.

\begin{definition}[\textbf{Realistic function class}]
  Let $\Delta \subseteq \R^d$ and $\Omega \subseteq \R^K$ be compact sets.  A
  class $\mathcal{F}$ of function from $\Delta$ to $\Omega$ is said to be a
  \emph{$(p,J)$-realistic} function class if $\F$ admits a
  $J$-Lipschitz-parametrization by $p$ parameters. Formally, there exists a
  compact set $B_p \subseteq \R^p$ and a map $\tau: B^p \rightarrow \F$ such
  that for all $w_1, w_2 \in B_p$ and all $x \in \Delta$,
\begin{equation}\label{eq:lipparam}
  \norm[2]{\tau(w_1)(x) - \tau(w_2)(x)} \leq J \norm[2]{ w_1 - w_2}.
\end{equation}
The set $B_p$ is called the \emph{parameter domain} of $\F$, the set $\Delta$ is
called the \emph{input domain} of $\F$, and the set $\Omega$ is called the
\emph{co-domain} of $\F$.
\end{definition}

We now observe that both regression and classification for bounded domains using
neural networks with bounded parameters can be modeled using $(p,J)$-realistic
function classes.

\begin{example}[Neural networks for regression.]\label{ex:reg}
Let $B_p$ be a subset of $\R^p$ bounded in some unspecified norm and $\mathcal{F}_1$ be the class of $p$-parameter neural networks with parameters in $B_p$. Let $\mathcal{X}$ be a bounded domain for the covariates $x \in \R^d$. Then, as in \cite{BSisoperimetry}, the natural map mapping parameter space to neural networks $\psi_1 : B_p \rightarrow \mathcal{F}_1$ satisfies \eqref{eq:lipparam} for all $w_1,w_2 \in \R^p$ and $x \in \mathcal{X}$, for some choice of $J$.
\end{example}

\begin{example}[Neural networks for classification.]\label{ex:class}
Let $B_p$ be a subset of $\R^p$ bounded in some unspecified norm and $\mathcal{X}$ be a bounded domain for covariates. Construct $\mathcal{F}_2$ by applying the $\sf$ operator on $\mathcal{F}_1$ as in Example \ref{ex:reg} (hence $\mathcal{F}_2$ has range in $\Delta^K$). Then, since $\sf$ is Lipschitz (specifically with Lipschitz constant as 1 for $\ell_2$ norms on the input and output), $\sf \circ \psi_1$ gives a $J$-Lipschitz parameterization for $\mathcal{F}_2$.
\end{example}

\paragraph{Concentration of measure}
We will need the notion of \emph{sub-Gaussian} random variables.  A random
variable $X$ with mean $\mu$ is said to have \emph{sub-Gaussian} parameter
$\sigma > 0$ if for all real $\lambda$, it holds that
$\E[\exp(\lambda (X - \mu))] \leq \lambda^2\sigma^2/2$.  Thus, if
$X_1, X_2, \dots, X_n$ are independent sub-Gaussian random variables with
parameters $\sigma_1, \sigma_2, \dots, \sigma_n$, then their sum is also
sub-Gaussian with parameter $\sqrt{\sum_{{i=1}}^n\sigma_i^2}$.  If $X$ is
sub-Gaussian with parameter $\sigma$ then the \emph{Hoeffding inequality} states
that for all $t > 0$, it holds that
\begin{equation}
  \Pr\insq{X \leq \E[X] - t} \leq \exp\inp{-\frac{t^2}{2\sigma^2}}. \label{eq:hoeffding-general}
\end{equation}
It is also well known that if a random variable $X$ has support in the interval
$[a, b]$, then it is sub-Gaussian with parameter $\frac{b-a}{2}$.  Combined with
the above discussion this leads to the usually stated form of Hoeffding's
inequality: if $X_1, X_2, \dots X_n$ are i.i.d. copies of such an $X$ then for
all $t > 0$,
\begin{equation}
  \label{eq:hoeffding}
  \Pr\insq{\frac{1}{n}\sum_{i=1}^nX_i < \E[X]  -t} \leq \exp\inp{-\frac{2nt^2}{(b-a)^2}}.
\end{equation}
It is well known that the assumption that \(X_1, X_2, \dots, X_n\) are i.i.d can
be relaxed so that one gets the following result (usually called \emph{Azuma's
  inequality}): if \(Y_0 = 0, Y_1, Y_2, \dots, Y_n\) form a martingale sequence
(with respect to some filtration) such that \(\abs{Y_i - Y_{i-1}} \leq c\) holds
with probability \(1\) for each \(1 \leq i \leq n \), then
\(\Pr\insq{Y_n \leq -nt} \leq \exp\inp{-\frac{nt^2}{2c^2}}\).  Via the Doob martingale
construction, this inequality leads to the \emph{bounded differences
  inequality}, one special case of which is the
following~\cite{yurinskiiExponentialBoundsLarge1974}: let
\(V_1, V_2, \dots V_n\) be independent mean zero random vectors in \(\R^d\) such
that \(\norm{V_i} \leq b\) holds with probability \(1\) for \(1 \leq i \leq n\).  Then
for every \(t > 0\),
\(\Pr\insq{\norm{\sum_{i=1}^nV_i} \geq \E\insq{\norm{\sum_{i=1}^nV_i}} + nt} \leq
\exp\inp{-nt^2/(8b^2)}.\) Using the fact that the \(V_i\) are independent and
have mean zero, one has \(\E\insq{\norm{\sum_{i=1}^nV_i}} \leq b\sqrt{n}\).  Combining
this with a bit of algebra, the above inequality can be simplified to the
following:\footnote{See \cite{yurinskiiExponentialInequalitiesSums1976} and
  \cite[Section 6.3]{ledouxProbabilityBanachSpaces1991} for a history of
  related, more sophisticated inequalities and
  \cite{hayesLargedeviationInequalityVectorvalued2003} and \cite[Theorem
  12]{grossRecoveringLowRankMatrices2011} for related statements.}
\begin{equation}
  \label{eq-vector}
  \Pr\insq{\norm{\frac{1}{n}\sum_{i=1}^nV_i} \geq t} \leq  2\exp\inp{-\frac{nt^2}{16b^2}}
  \text{ for all \(t \geq 0\)}.
\end{equation}

We also record the following well-known fact.
\begin{fact}\label{fct-multiply-sub-gauss}
  There exists a positive constant $C$ such that the following is true. If $X$
  is a mean-zero random variable with sub-Gaussian parameter $\sigma$, and $Z$
  is any random variable (not necessarily independent of $X$) such that
  $\E[ZX] = 0$ and $\abs{Z} \le M$ a.s.{}, then $ZX$ has sub-Gaussian parameter
  at most $C M\sigma$.
\end{fact}
All of the above facts are standard, and those for which no reference has been
provided above can be found, e.g., in \cite[Chapter
2]{Wainwright_2019}.  In particular, a proof of the previous fact is implicit in
the calculations in \cite[pp.~46--47]{Wainwright_2019}.

The following important notion (isolated in this form by Bubeck and
Sellke~\cite{BSisoperimetry}) will play an important role.
\begin{definition}[\textbf{$c$-isoperimetry}]\label{def:isoperimetry}
  Given $c > 0$, a distribution $\D$ on $\R^d$ is said to satisfy
  \emph{$c$-isoperimetry} if for very $L$-Lipschitz function
  $f: \R^d \rightarrow \R$, the random variable $f(X)$, when $X$ is sampled
  according to $\D$, is sub-Gaussian with parameter $L\sqrt{c/d}$.
\end{definition}
The important factor to note here is the dimension-dependent factor of
$1/\sqrt{d}$ in the sub-Gaussian parameter.  Roughly speaking, this says that
the distribution $\D$ exhibits a concentration of measure phenomenon similar to
the \emph{normalized} $d$-dimensional standard Gaussian distribution (i.e., with
mean $0$ and co-variance matrix $\frac{1}{d}\mathbf{I}$).  For further
discussion on this assumption, see \cite{BSisoperimetry}. Our results extend to
mixtures of isoperimetric distributions as in \cite{BSisoperimetry}.
\subsection{Overfitting}
We are now ready to define our notion of overfitting to the training set.  Our
starting point is the optimality of the conditional expectation as an estimator
(\cref{thm-best-bregman}).  In particular, from \cref{exsec} in
\cref{thm-best-bregman}, we can conclude that $\E[D_{\phi}(Y,f(X))]$ is at least
$\sigma^{2}_{\phi} \defeq E[D_{\phi}(Y,\E[Y|X])]$

We thus say that $f$ $\epsilon$-overfits training data
$\{x_i,y_i\}^n_{i=1}$ if
 \begin{equation}
   \label{eq:overfitting}
   \frac{1}{n}\sum^n_{i=1} D_{\phi}(y_{i},f(x_{i}))
   < \sigma^{2}_{\phi} -
   \epsilon,
 \end{equation}
 i.e., if the empirical Bregman divergence loss is lower (by at least
 $\epsilon$) than the minimum, over all functions
 $f: \R^d \rightarrow \mathbb{R}^{K}$, of the expected Bregman divergence loss where $x_i\in \R^{d}$ and
 $y_i\in \Omega$. This notion of overfitting for Bregman divergence loss
 generalizes that for the square loss used in \cite{BSisoperimetry}; in
 particular, when $D_{\phi}$ is the square loss, it reduces to their
 notion.
 \section{Proof of the main theorem}
 \label{sec:proof-main-theorem}
We are now ready to state the main result of our paper.  We assume that the
model for generating the covariates $(X, Y)$ is described by the following
graphical model
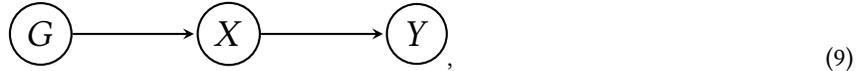
\begin{equation}
  \label{eq:mixture}
  \begin{tikzpicture}[->,>=stealth,shorten >=1pt,auto,node distance=2.5cm,
    thick,main node/.style={circle,draw,font=\sffamily\Large\bfseries}]

\node[main node] (1) {$G$}; \node[main node] (2) [right of=1]
    {$X$}; \node[main node] (3) [right of=2] {$Y$};

\path[every node/.style={font=\sffamily\small}]
    (1) edge (2)
    (2) edge (3);
  \end{tikzpicture},
\end{equation}
where $G$ denotes the label of the mixture component.  In particular, we assume
that the label $Y$ is independent of the index of the mixture component,
conditioned on the covariate $X$.  This is the same model as the one used by
Bubeck and Sellke~\cite[Theorem 3, points 2 and 3]{BSisoperimetry}. (See also
\cref{sec:extension-mixtures}.)
\begin{theorem}\label{maintheorem}
  Let $\Omega$ be a compact convex subset of $\R^K$ for some
  $K > 0$, with $\ell_{\infty}$-diameter at most $\omegadiam$.  Let
  $\phi: \Omega \rightarrow \R$ be a continuously differentiable strictly convex
  function. Let $D_{\phi}$ denote the corresponding Bregman divergence loss.
  For $\Delta \subseteq \R^{d}$, let $\D$ be a probability distribution on
  $\Delta \times \Omega$ such that $\D$ obeys the graphical model in
  \cref{eq:mixture} and such that the marginal $\D_{X}$ of $\D$ is a mixture of
  $\compnum$ distributions $\inp{\D_i}_{i=1}^{\compnum}$ each of which is
  $c$-isoperimetric for some $c > 0$.  Assume that $\phi$ satisfies the
  following regularity condition: there exists a subset $A\subseteq \Omega$, such that a version of
    $\E_{(X, Y)\sim \D}[Y|X]$ takes values only in $A$, and such that \(A\) satisfies
    the following:\footnote{Since $\Omega$ is a compact convex subset of
      $\R^{K}$, it can be assumed without loss of generality that a version of
      $\E[Y|X]$ taking values only in $\Omega$ exists.}
    \begin{align}
      a_0 & \defeq \sup_{a \in A} \norm{a}
      \leq m_{0} \defeq \max_{\omega \in \Omega}\norm{\omega} < \infty, \text{ and } \label{eq:A-bounded}\\
      m_{2} & \defeq \sup_{a\in A} \abs{\phi(a)}
      \leq m_1 \defeq \max_{\omega \in \Omega} \abs{\phi(\omega)} < \infty \text{, and }\label{eq:phi-bounded}\\
      m_{3} &\defeq \sup_{a\in A} \norm{\nabla \phi(a)} <\infty \label{eq:phi-bounded-derivative}.
    \end{align}
Let $\mathcal{F}$ be a $(p, J)$-realistic function class with parameter domain
  $B_p$, input domain $\Delta \subseteq \R^d$, and co-domain $\Omega$.  Assume
  that
  \begin{enumerate}
  \item $\phi$ is $\phiL$-Lipschitz on the range
    $R \defeq \inb{f(x) | x \in \Delta \text{ and } f \in \F}$ of $\F$, and
  \item For each $1 \leq \ell \leq K$, the derivative $(\nabla \phi)_{\ell}$ of the
    $\ell$\textsuperscript{th} coordinate of $\phi$ is \phidL{}-Lipschitz on the range
    \(R\).  \label{item:phi-grad-lip}
  \end{enumerate}
  Further define
  \begin{equation}
    \label{eq:global-bounds}
    W \defeq \operatorname{diam}(B_{p}) \text{ and }
    \gamma \defeq \sup_{\substack{f \in \F \\ x \in \Delta}} \norm[2]{\nabla\phi(f(x))}.
  \end{equation}
  Given $\epsilon, \delta \in (0, 1)$, fix a positive integer $n$ satisfying
\begin{equation*}
    n \geq \max\inb{\frac{300\log(\frac{10K}{\delta})}{\epsilon^2}
      \inp{
        m_1 + m_2 + 2\max\inb{3 \gamma, m_3}(m_0 + a_0)}^2,
      \frac{2048 K^2\gamma^{2}\compnum\omegadiam^2\log(\frac{10K\compnum}{\delta})}{\epsilon^2}}.
  \end{equation*}
Let $(X_i, Y_i)_{i=1}^n$ be $n$ i.i.d samples from $\D$.  Then, with
  probability at least $1 - \delta$ over the random choice of these samples, the
  Lipschitz constant $L$ of any function $f \in \F$ that $\epsilon$-overfits
  (see \cref{eq:overfitting}) these samples satisfies
\begin{equation}
  \label{eq:main-bound}
  L\geq \frac{\epsilon}{32CK\omegadiam\phidL\sqrt{2c}}
  \sqrt{
    \frac{
      nd
    }{
      p\log (1+8JW(\omegadiam \phidL {K} + \phiL +\gamma)/\epsilon)
      + \log(5K/\delta)
    }
  }.
\end{equation}
\end{theorem}
Note that when $\phi$ is a $C^{2}$ function then we can take
$\phidL = \max_{x \in \Omega} \norm[{2 \rightarrow 2}]{\nabla^{2} \phi(x)}$.
Almost all losses used in machine learning are $C^{2}$
losses, for example, square loss, cross-entropy defined on a bounded set away
from the $\vec{0}$ vector and coordinate axes.
\paragraph*{Remark}  \cref{maintheorem} tells us that when a Bregman divergence loss (such as cross-entropy loss, logistic loss, etc.) is used for training, overparameterization becomes essential for robust interpolation (i.e., achieving a low Lipschitz constant). Note that all the regularity assumptions of \cref{maintheorem} are satisfied by the cross-entropy loss, provided that the probability of each label given the covariate is bounded away from zero. For $K$-class classification, the assumption $\Pr[Y_{i} = 1 \mid X = x] \geq \alpha>0$ for all $i$ and all $x$ implies the existence of  $A \subseteq \tilde{\Delta}_{K} \cap [\alpha, 1 - \alpha)^K$, where $\Omega = \tilde{\Delta}_{K}$ denotes the $K$-dimensional probability simplex. For further details, please refer to \cref{cor-classification}. 

We now proceed to describe the steps of the proof of \cref{maintheorem}, and
begin by setting up some notational conventions. We denote the sample
$(X_i, Y_i)$ by $S_i$ for $1 \leq i \leq n$.  Given a function $f \in \F$, we
let $Z_i$ denote the random variable $D_{\phi}(Y_i, f(X_i))$ (the function $f$
would be understood from the context).  When discussing a single sample, we will
often drop the index and denote these as $S = (X, Y)$, and
$Z = D_{\phi}(Y, f(X))$.  Similarly, we will use the corresponding small case
letters $s_i, s, x_i, x$, etc{}. to denote specific realizations of the
corresponding random variables $S_i, S, X_i, X$, etc.  The starting point of the
proof is the following simple but important decomposition.
\begin{lemma}[Decomposition]\label{lem:decompose}
  Fix $f \in \F$, and with the above notation, define (in accordance with
  \cref{eq:overfitting})
  \begin{equation}
    \label{eq:6}
    \sigma_{\phi}^2 \defeq \mathop{\E}\limits_{(X, Y) \sim \D}[D_{\phi}(Y,
    \E[Y|X])].
  \end{equation}
  We then have
  \begin{equation}
    \label{eq:z-decompose}
    Z - \sigma_{\phi}^2 = \Phi_1 + \Phi_2 + \sum_{i=1}^3\Gamma_i,
  \end{equation}
  where $\Phi_1, \Phi_2$, and $\inp{\Gamma_i}_{i=1}^3$ are random variables
  defined as
  \begin{gather}
    \Phi_{1}
     \defeq D_{\phi}(\E[Y|X], f(X)),\;
    \Phi_{2}
     \defeq D_{\phi}(Y,\E[Y|X]) - \sigma^{2}_{\phi},\\
\Gamma_{1}
    \defeq \ina{Y - \E[Y|X], \nabla\phi(\E(Y|X))}, \;
    \Gamma_{2} = \Gamma_2(f)
    \defeq -\ina{Y - \E[Y|X], \E\left[\nabla \phi (f(X))\right]}, \text{ and }\\
    \Gamma_{3} = \Gamma_3(f)
    \defeq -\ina{Y - \E[Y|X], \nabla \phi (f(X)) - \E\left[\nabla \phi (f(X))\right]}.
  \end{gather}
  Further, $\Phi_2$ and the $\Gamma_i$ ($1 \leq i \leq 3$) have mean $0$, and
  $\Phi_1$ is non-negative.
\end{lemma}
\paragraph*{Remark} As discussed earlier, our proof technique relies on a bias-variance type decomposition of Bregman divergence losses. In the decomposition \cref{lem:decompose}, the terms $\Phi_1$ and $\Phi_2$ correspond to the bias and variance components, respectively, while $\Gamma_1$, $\Gamma_2$, and $\Gamma_3$ are mean-zero terms. Specifically, the term $\Gamma_3$ involves the function $f$ evaluated at random points. 
\begin{proof}
  The decomposition follows by applying the triangle decomposition for the
  Bregman divergence (\cref{eq:br-triangle}) to $Z = D_{\phi}(Y, f(X))$ with the
  ``third point'' chosen as $\E[Y|X]$.  The non-negativity of $\Phi_1$ follows from
  the non-negativity of Bregman divergences.  $\E[\Phi_2]$ is zero by the
  definition of $\sigma_{\phi}^2$ (\cref{eq:overfitting,eq:6}).  We also have,
  for example,
  \begin{align}
    \E[\Gamma_{3}]
    &= \E[\ina{Y - E[Y|X], \nabla \phi (f(X)) - \E[\nabla \phi (f(X))]}] \\
    &= \E[\E[\ina{Y - \E[Y|X], \nabla \phi (f(X)) - \E[\nabla \phi (f(X))]}|X] \\
    &\overset{(\star)}{=} \E[\ina{\E[Y - \E[Y|X]|X], \nabla \phi (f(X)) - \E[\nabla \phi (f(X))]}] = 0,\label{eq:2}
  \end{align}
  where the equality marked $(\star)$ follows because
  $\nabla \phi (f(X)) - \E[\nabla \phi (f(X))]$ is measurable with respect to
  the $\sigma$-field generated by $X$.  The computations for $\Gamma_1$ and
  $\Gamma_2$ are similar.
\end{proof}

We now prove appropriate concentration results for sample estimates of
$\Phi_2$ and the $\Gamma_i$, beginning with $\Phi_2$.
\begin{observation}\label{obv-Phi-2}
  Let $\Phi_2^{(1)}, \Phi_2^{(2)}, \dots, \Phi_2^{(n)}$ be $n$ i.i.d.{} samples of $\Phi_2$
  generated using $n$ i.i.d.{} samples $(X_i, Y_i)_{i=1}^n$ from $\D$.  Let
  $M_0 \defeq \sup_{(x, y) \in \Delta \times \Omega} D_{\phi}(y, \E_{(X, Y) \sim
    \D}[Y|X=x]).$ Then for any $\epsilon > 0$.
  \begin{equation}
    \Pr\inp{\frac{1}{n}\sum^n_{i=1}\Phi_{2}^{(i)} \leq -\epsilon} \leq  e^{-2n\epsilon^{2}/M_0^{2}}.
  \end{equation}
  Further, in the notation of
  \cref{eq:A-bounded,eq:phi-bounded,eq:phi-bounded-derivative} in
  \cref{maintheorem}, we can take $M_0 = m_1 + m_2 + m_3(m_0 + a_0)$. \end{observation}
\begin{proof}
  The concentration claim follows directly from the Hoeffding inequality applied
  to the $n$ i.i.d.{} random variables $D_{\phi}(Y_i, \E[Y_{i}|X_{i}])$,
  $1 \leq i \leq n$, which are all constrained to lie in the interval
  $[0, M_0]$. The bound on $M_0$ follows as given below. For $p \in \Omega$ and
  $q \in A \subseteq \Omega$ (with $\Omega$ and $A$ as defined in
  \cref{maintheorem}),
\begin{align}
  D_{\phi}(p, q)
  &= |\phi(p)-\phi(q)-\langle  \nabla \phi(q)\; , \; p-q \rangle|\\
  &\leq |\phi(p)|+|\phi(q)|+ \norm{\nabla\phi(q)}
    \cdot \inp{\norm{p} + \norm{q}}\\
  &\leq m_{1}+m_{2}+m_{3}\cdot\inp{m_{0}+a_{0}}. \qedhere
\end{align}
\end{proof}
An analogous application of the Hoeffding bound gives the following.
\begin{observation}\label{obv-Gamma-1}
  Let $\Gamma_1^{(1)}, \Gamma_1^{(2)}, \dots, \Gamma_1^{(n)}$ be $n$
  i.i.d.{} samples of $\Gamma_1$ generated using $n$ i.i.d.{} samples
  $(X_i, Y_i)_{i=1}^n$ from $\D$.  Then for any $\epsilon > 0$.
  \begin{equation}
    \Pr\inp{\frac{1}{n}\sum^n_{j=1}\Gamma_1^{({j})} \leq -\epsilon} \leq  e^{-2n\epsilon^{2}/M_1^{2}},
  \end{equation}
  where, in the notation of
  \cref{eq:A-bounded,eq:phi-bounded,eq:phi-bounded-derivative}
in \cref{maintheorem}, we can take $M_1 = 2 m_3\cdot ( m_0 + a_0)$. \end{observation}

The proof for $\Gamma_2$ is also similar, since it does not depend upon
evaluations of $f$ on a random point.
\begin{observation}\label{obv-Gamma-2}
  For $f \in \F$, let $\Gamma_2^{(1)}(f), \Gamma_2^{(2)}(f), \dots, \Gamma_2^{(n)}(f)$ be
  $n$ i.i.d.{} samples of $\Gamma_2$ generated using $n$ i.i.d.{} samples
  $(X_i, Y_i)_{i=1}^n$ from $\D$.  Then for any $\epsilon > 0$.
  \begin{equation}
    \Pr\inp{\exists f \in \F \text{ s.t. } \frac{1}{n}\sum^n_{j=1}\Gamma_2^{({j})}(f) \leq -\epsilon} \leq 2 e^{-2n\epsilon^{2}/M_2^{2}},
  \end{equation}
  where, in the notation of
  \cref{eq:A-bounded,eq:phi-bounded,eq:phi-bounded-derivative,eq:global-bounds}
in \cref{maintheorem}, we can take $M_2 = 6 \gamma \cdot (m_0 + a_0)$.
\end{observation}
\begin{proof}
  Consider the mean-zero i.i.d random vectors $V_i \defeq Y_i - \E[Y_i | X_i]$
  where $1 \leq i \leq n$.  From \cref{eq:A-bounded}, we also have that
  $\norm{V_i} \leq m_0 + a_0$ for each $i$. Let $\gamma$ be as defined in \cref{eq:global-bounds}.  Applying the version of the
  bounded differences inequality from \cref{eq-vector} then gives
  \begin{align}
    \Pr\insq{\norm{\frac{1}{n}\sum_{i=1}^n(Y_i - \E[Y_i | X_i])} \geq \frac{\epsilon}{\gamma}}
&\leq 2\exp\inp{-\frac{2n\epsilon^2}{M_2^2}}.\label{eq:14}
  \end{align}
  Denote by $V$ the random vector
  $\frac{1}{n}\sum_{i=1}^n(Y_i - \E[Y_i | X_i])$.  Note that
  \begin{equation}
    \frac{1}{n}\sum_{i=1}^n\Gamma_2^{(i)}(f) = -\ina{V, \E[\nabla\phi(f(X))]}
    \geq -\norm{V}\cdot \norm{\E[\nabla\phi(f(X))]}. \label{eq:22}
  \end{equation}
  Further, for each $f\in\F$, we have, from \cref{eq:global-bounds}, that
  $\norm{\E[\nabla\phi(f(X))]} \leq \gamma$.  Combining this bound with
  \cref{eq:22,eq:14} gives the claim.
\end{proof}

Finally, we study $\Gamma_3$.  Note that of the terms in the decomposition, this
is the only term which depends upon the evaluation of a function $f$ from $\F$ at a
random input. It is here that the notion of $c$-isoperimetry
enters the picture.  For simplicity of presentation, we assume in the statement
of \cref{lem-Gamma-3} that the number $\compnum$ of mixture components is
\emph{one} (so that the marginal $\D_X$ is itself $c$-isoperimetric).  The full
proof for the case $\compnum > 1$ is presented in \cref{sec:extension-mixtures}.
\begin{lemma}\label{lem-Gamma-3}
  Fix $f \in \F$ as in \cref{maintheorem} and assume that $f$ is $L$-Lipschitz.
  Let $\Gamma_3^{(1)}(f), \Gamma_3^{(2)}(f), \dots, \Gamma_3^{(n)}(f)$ be $n$ i.i.d.{}
  samples of $\Gamma_3(f)$ generated using $n$ i.i.d.{} samples
  $(X_i, Y_i)_{i=1}^n$ from $\D$.  Then for any $\epsilon > 0$.
  \begin{equation}
    \Pr\insq{\frac{1}{n}\sum_{i=1}^n\Gamma_3^{(i)}(f) \leq -\epsilon}
    \leq K \exp\inp{-\frac{nd\epsilon^2}{2cC^2K^2\omegadiam^2L^2\phidL^2}}.
  \end{equation}
  where $C$ is the universal constant from \cref{fct-multiply-sub-gauss},
  $\omegadiam$ represents the $\ell_{\infty}$-diameter of $\Omega$, and the remaining notation
  comes from
  \cref{item:phi-grad-lip,eq:A-bounded,eq:phi-bounded,eq:phi-bounded-derivative,eq:global-bounds}
  in \cref{maintheorem}.
\end{lemma}
\begin{proof}
  Since $f$ remains fixed through out this proof, we drop the dependence of
  $\Gamma_3$ on $f$ from the notation. For $1 \leq i \leq n$, and
  $1 \leq \ell \leq K$, define
  \begin{equation}
    U_{i,\ell} \defeq
    -(Y_{i,\ell} - \E[Y_{i,\ell}|X_i])\cdot(\nabla\phi(f(X_i))_{\ell} - \E[\nabla\phi(f(X_i))_{\ell}]).\label{eq:1}
  \end{equation}
  Note that $\sum_{i=1}^n\Gamma_3^{(i)} = \sum_{\ell=1}^K\sum_{i=1}^nU_{i, \ell}$, so that by an union bound,
  \begin{equation}
    \Pr\insq{\frac{1}{n}\sum_{i=1}^n\Gamma_3^{(i)} \leq -\epsilon}
    \leq \sum_{\ell=1}^K\Pr\insq{\frac{1}{n}\sum_{i=1}^nU_{i,\ell} \leq -\epsilon/K}.\label{eq:4}
  \end{equation}

  Further, note that for any fixed $\ell \in [K]$, the random variables
  $U_{1,\ell}, U_{2, \ell}, \dots, U_{n,\ell}$ are i.i.d, and are also mean-zero
  (by the same argument as in \cref{eq:2}).  We now proceed to estimate their
  sub-Gaussian parameter.  From \cref{item:phi-grad-lip} in \cref{maintheorem},
  we see that the function $(\nabla\phi)_{\ell} \circ f: \R^d \rightarrow \R$ is
  $\phidL \cdot L$-Lipschitz.  From the $c$-isoperimetric assumption on the
  distribution of the $X_i$, it therefore follows that for each
  $i \in [n], \ell \in [K]$, the mean-zero random variable
  $V_{i, \ell} \defeq \nabla\phi(f(X_{i}))_{\ell} -
  \E[\nabla\phi(f(X_i))_{\ell}]$ is sub-Gaussian with parameter
  $\phidL{}L\sqrt{c/d}$.  Further, the random variables
  $T_{i, \ell} \defeq -(Y_{i,\ell} - \E[Y_{i,\ell}|X_{i}])$ satisfy
  $\abs{T_{i, \ell}} \leq \min\inb{\omegadiam, m_0 + a_0}$ (by
  \cref{eq:A-bounded} and the assumption on the diameter of $\Omega$).  From
  \cref{fct-multiply-sub-gauss}, it thus follows that the $U_{i, \ell}$ are all
  sub-Gaussian with parameter $C\omegadiam\phidL{}L\sqrt{c/d}$.  Applying the
  Hoeffding equality to each of the $K$ terms on the right hand side of
  \cref{eq:4} then gives
  \begin{equation*}
    \Pr\insq{\frac{1}{n}\sum_{i=1}^n\Gamma_3^{(i)} \leq -\epsilon}
    \leq K \exp\inp{-\frac{nd\epsilon^2}{2cC^2K^2\omegadiam^2L^2\phidL^2}}.\qedhere
  \end{equation*}
\end{proof}

We also need the following simple lemma.
\begin{lemma}\label{lem-epsilonnet}
  With the setup of \cref{maintheorem}, let $f$ and $g$ be functions in $\F$
  such that $\sup_{x \in \Delta}\norm{f(x) - g(x)} \leq \nu$.  Then for any
  $(x, y)$ in the support of $\D$, we have
  \begin{equation}\label{eq:8}
    \abs{D_{\phi}(y, f(x)) - D_{\phi}(y, g(x))}
    \leq \nu \cdot \inp{\omegadiam\phidL{K} + \phiL + \gamma }
  \end{equation}
\end{lemma}
\begin{proof}
  It is enough to prove the upper bound for
  $D_{\phi}(y, f(x)) - D_{\phi}(y, g(x))$ (the bound on the absolute value
  follows by interchanging the role of $f$ and $g$).  For this we again employ
  the triangle decomposition~\cref{eq:br-triangle}.
  \begin{equation}
    D_{\phi}(y, f(x)) - D_{\phi}(y, g(x)) = D_{\phi}(g(x), f(x))
    - \ina{y - g(x), \nabla\phi(f(x)) - \nabla\phi(g(x))}.
  \end{equation}
  Here, the first term is at most $\nu \cdot (\phiL{} + \gamma)$, while the
  second term is at most $\omegadiam \nu \phidL {K}$, where the parameters
  are as defined in the statement of \cref{maintheorem}.  The claim thus
  follows.
\end{proof}

We are now ready to prove \cref{maintheorem}.

\begin{proof}[Proof of \cref{maintheorem}] Given the above lemmas, the structure
  of the proof is similar to that in the work of Bubeck and
  Sellke~\cite{BSisoperimetry}.  The goal is to show that if $L$ is not too
  large, then with high-probability, no $L$-Lipschitz function in $\F$ can
  $\epsilon$-overfit the observed data.  \cref{lem-Gamma-3} essentially
  establishes this for any particular $f \in \F$, and to perform a ``union
  bound'' over the uncountable set $\F$, one needs to pass to an appropriate
  finite net.  We now proceed to the details.

  For a given $f \in \F$, let
  $\inp{\Phi_{1}^{(i)}}_{i=1}^n, \inp{\Phi_{2}^{(i)}}_{i=1}^n$,
  $\inp{\Gamma_{1}^{(i)}}_{i=1}^n$, $\inp{\Gamma_{2}^{(i)}(f)}_{i=1}^n$ and
  $\inp{\Gamma_{3}^{(i)}(f)}_{i=1}^n$ be i.i.d.{} sequences obtained by
  decomposing each $Z_i - \sigma_{\phi}^2$ according to \cref{eq:z-decompose}.

  Fix $L$ to be equal to the lower bound claimed in \cref{maintheorem}.  Let $\F_L \subseteq \F$
  denote the set of all $L$-Lipschitz functions in $\F$.  Now, since the events
  considered in \cref{obv-Phi-2,obv-Gamma-1,obv-Gamma-2} do not depend upon the
  choice of $f$, we have
  \begin{align}
    \Pr\inp{\exists f \in \F_L \text{ s.t. }
    \frac{1}{n}\sum^n_{i=1}\Phi_{2}^{({i})} \leq - \epsilon/8}
    & \leq  e^{-2n(\epsilon/8)^{2}/M_0^{2}}\text{, and }\label{eq:10}\\
    \Pr\inp{\exists f \in \F_L \text{ s.t. }
    \frac{1}{n}\sum^n_{i=1}\Gamma_1^{({i})} \leq - \epsilon/8}
    &\leq
      e^{-2n(\epsilon/8)^{2}/M_1^{2}}, \text{ and} \label{eq:11}\\
    \Pr\inp{\exists f \in \F_L \text{ s.t. }
    \frac{1}{n}\sum^n_{i=1}\Gamma_2^{({i})}(f) \leq - \epsilon/8}
    &\leq
      2Ke^{-2n(\epsilon/8)^{2}/M_2^{2}}.\label{eq:11-new}
  \end{align}
  Let the events on the LHS above be denoted $E_0, E_1$ and $E_2$ for future
  reference.  We now proceed to analyze $\Gamma_3$.  Recall that by a standard
  argument a $\nu/2$-net for $\F$ can be modified to form a $\nu$-net of $\F_L$
  of the same size, all of whose elements are also elements of $\F_L$.  Further
  an $\epsilon'$-net of $B_p$ maps (under the map $\tau$ in the definition of a
  $(p,J)$-realistic class) to a $J\epsilon'$-net for $\F$ (under the $\sup$
  norm).  Note also that by standard arguments $B_p$ has an $\epsilon'$-net of
  size at most $(1 + 2W/\epsilon')^p$.  Set
  $\nu \defeq \frac{\epsilon}{2(\omegadiam \phidL {K} + \phiL + \gamma)}$.
  Then, from the above arguments, we see that $\F_{L}$ has a $\nu$-net
  $\F_{L,\nu} \subseteq \F_L$ of size at most $(1 + 4WJ/\nu)^p \leq \exp(4pWJ/\nu)$.  Applying
  \cref{lem-Gamma-3} and taking a union bound over $\F_{L,\nu}$ then yields
  \begin{equation}
    \label{eq:5}
    \Pr\insq{\exists f \in \F_{L,\nu} \text{ s.t. }
      \frac{1}{n}\sum_{i=1}^n\Gamma_3^{(i)}(f) \leq - \epsilon/8
    }
    \leq K \abs{\F_{L,\nu}}
    \cdot \exp\inp{-\frac{nd(\epsilon/8)^2}{2cC^2K^2\omegadiam^2L^2\phidL^2}}.
  \end{equation}
  Let the event above be denoted $E_3$.  Now, from \cref{lem:decompose},
\begin{equation}
    \label{eq:7}
    \frac{1}{n}\sum_{i=1}^nZ_i(f) - \sigma_{\phi}^2
    = \frac{1}{n}\sum_{i=1}^n\Phi_1^{(i)}
    + \frac{1}{n}\sum_{i=1}^n\Phi_2^{(i)}
    + \frac{1}{n}\sum_{i=1}^n\Gamma_1^{(i)}
    + \frac{1}{n}\sum_{i=1}^n\Gamma_2^{(i)}(f)
    + \frac{1}{n}\sum_{i=1}^n\Gamma_3^{(i)}(f)
  \end{equation}
  Since the $\Phi_1^{(i)}$ are all non-negative, a union bound gives
  \begin{align}
    \label{eq:9}
    \Pr\insq{\exists f \in \F_{L,\nu} \text{ s.t. }
    \frac{1}{n}\sum_{i=1}^nZ_i(f) - \sigma_{\phi}^2 \leq - \epsilon/2
    } \leq \sum_{j=0}^3\Pr\insq{E_i},
  \end{align}
  where the $E_i$ are the events considered in \cref{eq:10,eq:11,eq:11-new,eq:5}
  above.  Finally, since $\F_{L,\nu}$ is a $\nu$-net for $\F_L$ under the $\sup$
  norm, it follows from \cref{lem-epsilonnet} and the choice of $\nu$ that
  (recall
  also that $Z_{i}(f) = D_{\phi}(Y_i, f(X_i))$)
\begin{equation}\label{eq:12}
    \Pr\insq{\exists f \in \F_{L} \text{ s.t. }
      \frac{1}{n}\sum_{i=1}^nZ_i(f) - \sigma_{\phi}^2 \leq - \epsilon
    } \leq
    \Pr\insq{\exists f \in \F_{L,\nu} \text{ s.t. }
      \frac{1}{n}\sum_{i=1}^nZ_i(f) - \sigma_{\phi}^2 \leq - \epsilon/2
    }.
    \end{equation}
Combining \cref{eq:12} with \cref{eq:9} and using
    \cref{eq:10,eq:11,eq:11-new,eq:5}, we thus get
  \begin{equation}
    \Pr\insq{\exists f \in \F_L \text{ s.t. }
      \frac{1}{n}\sum_{i=1}^nZ_i(f) - \sigma_{\phi}^2 \leq - \epsilon
    }
    \leq  K \abs{\F_{L,\nu}}
    \cdot e^{-\frac{nd(\epsilon/8)^2}{2cC^2K^2\omegadiam^2L^2\phidL^2}}
    +  2K \sum_{j=0}^2 e^{-2n(\epsilon/8)^{2}/(M_j)^{2}}.\label{eq:30}
  \end{equation}
  Using the size bound derived above for $\F_{L,\nu}$, and the choice of $L$ and
  $n$, we conclude that the right-hand side is at most $\delta$, since each term
  is at most $\delta/4$. \end{proof}

\section{Specializing the result to specific losses}
\label{sec:spec-result-spec}
In this section we show how \cref{maintheorem} leads to laws of robustness for
specific losses. We first verify that we can obtain the result of Bubeck and
Sellke~\cite{BSisoperimetry} for the square loss as a special case.  The calculations underlying the following verifications can be found in \cref{sec:proofs-coroll-main}.

Consider the regression setting using the mean squared error (MSE) loss. Assume
that $n$ input covariates and labels $((x_i, y_i))_{i=1}^n$ in
$\R^d \times [-M,M]^{K}$ are drawn from the distribution $\mathcal{D}$. Let the
hypothesis class for the regression problem be a $(p,J)$-realistic function
class $\F=\{f:\R^{d}\rightarrow [-M,M]^{K}\} \hspace{2pt}$ with parameter domain
$B_p$ having diameter $W$.  We specialize \cref{eq:overfitting} to this setting
(using \cref{item:sq-loss} of \cref{ex:bregman}),
and say that a function $f \in \F$ $\epsilon$-overfits the data if
\begin{equation}
  \label{eq:3}
  \frac{1}{n}\sum_{i=1}^{n}(f(x_{i})-y_{i})^{2}\leq \sigma^{2}-\epsilon,
\end{equation}
where $\sigma^{2}_{\phi}=\sigma^{2}=\E[\operatorname{Var}[Y|X]]$.

\begin{corollary}[\textbf{Law of robustness for
    regression}]\label{regressionspecial case}  Consider the above regression setting, and assume that for some
  $c>0$, the marginal $\D_{X}$ of $\D$ on the covariates is a mixture of
  $\compnum$ $c$-isoperimetric distributions.

  Given $\epsilon, \delta \in (0, 1)$, assume that the number $n$ of samples
  satisfies
  $n\geq \left(C_1 M^{4} K^{3}r
    \log\left(\frac{10Kr}{\delta}\right)\right)/\epsilon^{2},$ where $C_1$ is an
  absolute constant.  Then, with probability at least $1 - \delta$ over the
  samples, the Lipschitz constant $L$ of any function $f \in \F$ that
  $\epsilon$-overfits the samples satisfies
   \begin{equation}
     L\geq \frac{\epsilon}{128 CKM\sqrt{2c}}
  \sqrt{
    \frac{
      nd
    }{
      p\log (1+64JWKM/\epsilon)
      + \log(5K/\delta)
    }
  }.
   \end{equation}\label{cor-regression}
\end{corollary}
As another example, we consider the classification problem with a suitable loss.
Consider the $K$-class classification setting with the cross-entropy loss.
Assume that $n$ input covariates and labels $((x_i, y_i))_{i=1}^n$ in
$\R^d \times \inb{0, 1}^K$ are drawn from the distribution $\mathcal{D}$: here we
assume that the labels are one-hot encoded as $K$-dimensional binary
vectors. Let the hypothesis class for the classification problem be a
$(p,J)$-realistic function class
$\F=\{\operatorname{Softmax}(g) | g:\R^{d}\rightarrow [-M,M]^{K} \hspace{2pt}\}$
such that there exists a compact set $B_p \subseteq \R^p$ and a $J-$Lipschitz
map $\tau: B^p \rightarrow \F$, where $W=\operatorname{diam}(B_p)$.  Again, we
specialize \cref{eq:overfitting} to this setting (using \cref{item:kl-div} of
\cref{ex:bregman}) and say that a function $f \in \F$ $\epsilon$-overfits the
data if
\begin{equation}
  \label{eq:23}
  \frac{1}{n}\sum^n_{i=1} \sum^K_{\ell=1} -\ind{y_i=\ell} \log(f(x_i)_{\ell})\leq
  \sigma^{2}_{\phi}-\epsilon,
\end{equation}
where $\sigma^{2}_{\phi}=H(Y|X)$ is the conditional entropy of $Y$ given $X$.
In addition to the positivity of $\sigma_{\phi}^2$, we also need another
regularity condition on $\D$, which is that the probability of each label is
bounded away from $0$, even conditioned on the covariate: the number of samples
needed for our result then has a mild poly-logarithmic dependence on this lower
bound.

\begin{corollary}[\textbf{Law of robustness for classification}] Consider the
  above $K-$class classification setting, and assume that for some $c > 0$, the
  marginal $\D_{X}$ of $\D$ on the covariates is a mixture of $\compnum$
  $c$-isoperimetric distributions.  We further assume the regularity condition
  that there exists an $\alpha > 0$ such that $\Pr[Y_i=1|X=x] \geq \alpha$ for
  all $i \in [K]$ and $x \in \R^d$ (recall that the label $Y$ is a one-hot
  encoded vector).

  Define $a_{0} \defeq 1+ 2M + \log K$.  Given $\epsilon, \delta \in (0, 1)$,
  assume that the number $n$ of samples satisfies
  \[n \geq \frac{C_1
      K^{3}r\log\left(\frac{10Kr}{\delta}\right)}{\epsilon^{2}}
    \cdot \max\inb{a_{0}, 1 + \abs{\log \alpha}}^{2},
  \]
  where $C_1$ is an absolute constant.  Then, with probability at least
  $1 - \delta$ over the random choice of these samples, the Lipschitz constant
  $L$ of any function $f$ for which $\operatorname{Softmax}(f) \in \F$ and which
  $\epsilon$-overfits these samples (according to \cref{eq:23}) satisfies
  \begin{equation}
  L\geq \frac{\epsilon}{64CK\sqrt{2c}}
  \sqrt{
    \frac{
      nd
    }{
      p\log (1+8JW( e^{2M}K^{2} + \sqrt{K}(1 + 2M + \log K))/\epsilon)
      + \log(5K/\delta)
    }
  }.
\end{equation}\label{cor-classification}
\end{corollary}
\paragraph*{Remark}
A direct application of the \cref{maintheorem} gives only a lower bound on the Lipschitz constant of interpolating estimators of the form $\operatorname{Softmax}(f)$. However, in this case, a more appropriate quantity might be the Lipschitz constant of the function $f$, before the $\operatorname{Softmax}$ layer is applied. Our direct proof adapts to this and gives a better bound (by a factor of $Ke^{2M}$) for the Lipschitz constant of $f$ whenever $\operatorname{Softmax}(f)$ is interpolating. For details, please look at the proof of \cref{cor-classification} in the \cref{sec:proofs-coroll-main}.

\section{Extension to mixtures}\label{sec:extension-mixtures}

\emph{Mixture model}: We assume that the model for generating the covariates
$(X, Y)$ is described by the following graphical model
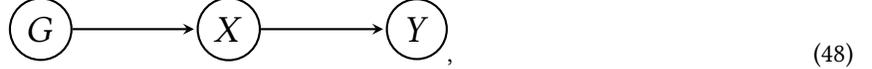
\begin{equation}
  \label{eq:mixture-again}
  \begin{tikzpicture}[->,>=stealth,shorten >=1pt,auto,node distance=2.5cm,
    thick,main node/.style={circle,draw,font=\sffamily\Large\bfseries}]

\node[main node] (1) {$G$}; \node[main node] (2) [right of=1]
    {$X$}; \node[main node] (3) [right of=2] {$Y$};

\path[every node/.style={font=\sffamily\small}]
    (1) edge (2)
    (2) edge (3);
  \end{tikzpicture},
\end{equation}
where $G$ denotes the label of the mixture component.  In particular, we assume
that the label $Y$ is independent of the index of the mixture component,
conditioned on the covariate $X$.  This is the same model as the one used by
Bubeck and Sellke~\cite[Theorem 3, points 2 and
3]{BSisoperimetry}.\footnote{See also the computation leading to and following
  eq. (2.5) in \cite{BSisoperimetry}, where this assumption has been used to
  (implicitly) deduce that the random variable in the left hand side of their
  eq.~(2.5) has mean $0$: this can fail if $Y$ is not independent of the mixture
  component when conditioned on $X$.}

The extension of the proof of \cref{maintheorem} given in
\cref{sec:proof-main-theorem} to the case of $\compnum > 1$ mixture components
has a structure similar to the similar extension by Bubeck and
Sellke~\cite{BSisoperimetry}; however, we provide the details for completeness.
The main technical step is replacing \cref{lem-Gamma-3} by a more general
analog.  Towards this end, we proceed to set up some notation.  In the
following, we also import all the notation from the statements of
\cref{maintheorem,lem:decompose}.

\newcommand{\hV}{\hat{V}} \newcommand{\tV}{\tilde{V}} Let $f \in F$ be given. For $1 \leq i \leq n$ and $1 \leq \ell \leq K$, define
the following random variables.
\begin{align}
  \label{eq:13}
  T_{i, \ell} &\defeq -(Y_{i,\ell} - \E[Y_{i,\ell}|X_{i}]),\\
  V(f)_{i, \ell} &\defeq \nabla\phi(f(X_{i}))_{\ell} - \E[\nabla\phi(f(X_i))_{\ell}],
  \\
  \hV(f)_{i, \ell} &\defeq \nabla\phi(f(X_i))_{\ell} -
                  \E[\nabla\phi(f(X_i))_{\ell} | G_i], \text{ and } \\
  \tV(f)_{i, \ell} &\defeq \E[\nabla\phi(f(X_i))_{\ell} | G_i]
                  - \E[\nabla\phi(f(X_i))_{\ell}],
  \end{align}
  and set
  \begin{equation}
    U(f)_{i,\ell} \defeq T_{i, \ell}V(f)_{i,\ell}
    = T_{i, \ell}\hV(f)_{i,\ell} +  T_{i, \ell}\tV(f)_{i,\ell}.\label{eq:20}
  \end{equation}
  Note that
  \begin{equation}
    \sum_{i=1}^n\Gamma_3^{(i)}(f) = \sum_{\ell=1}^K
    \sum_{i=1}^nU(f)_{i, \ell}.\label{eq:26}
  \end{equation}
Note that for any fixed $\ell \in [K]$ and $f \in \F$, the random variables
  $U(f)_{1,\ell}, U(f)_{2, \ell}, \dots, U(f)_{n,\ell}$ are i.i.d, and are also
  mean-zero (by the same argument as in \cref{eq:2}).

  To further study the distribution of the $U_{i,\ell}$, we will use the random
  variables $G = (G_i)_{i=1}^n$, taking values in $[\compnum]$, which denote the
  mixture component distribution $\D_{G_i}$ from which the $i$th covariate $X_i$
  is sampled.  Note that the $G_i$ are i.i.d.

  Note that we have the uniform bound
  $\abs{T_{i, \ell}} \leq \min\inb{\omegadiam, m_0 + a_0}$ (by
  \cref{eq:A-bounded} and the assumption on the diameter of $\Omega$).  We also
  note that
  \begin{equation}
    \E[Y_{i,\ell} | G] = \E[Y_{i,\ell}|G_i]
    = \E[\E[Y_{i,\ell}|G_{i}, X_{i}]|G_i] = \E[\E[Y_{i,\ell}|X_i]|G_i],\label{eq:15}
  \end{equation}
  where the first equality follows from the independence of the samples, the
  second from the tower property of conditional expectation, and the last from
  the conditional independence of $Y_i$ from $G_i$ given $X_i$. Combined again
  with the independence of the samples, \cref{eq:15} implies that
  \begin{equation}
    \label{eq:16}
    \E[T_{i,\ell} | G] =   \E[T_{i, \ell}| G_i]  = 0.
  \end{equation}
  We now have the following two lemmas.
  \begin{lemma}
    \label{lem-vhat-conc}
    With the notation above, we have, for every $L$-Lipschitz $f \in \F$ every
    $1 \leq \ell \leq K$, and every $\epsilon > 0$
   \begin{equation}
     \Pr\insq{\frac{1}{n}\sum_{i=1}^nT_{i,\ell}\hV(f)_{i,\ell} \leq -\epsilon}
     \leq \exp\inp{-\frac{nd\epsilon^2}{2cC^2\omegadiam^2L^2\phidL^2}}. \label{eq:21}
   \end{equation}
 \end{lemma}
 \begin{proof}
   Since $f$ remains fixed through the proof of the lemma, we drop the
   dependence of $\hV$ on $f$ from the notation.  We first note that
   $\E[\hV_{i,\ell} | G_i] = 0$, and also that $\hV_{i,\ell}$ is measurable with
   respect to the $\sigma$-field generated by the random variables $X_i$ and
   $G_i$.  Note also that
  \begin{equation}
    \E[T_{i,\ell}| G_i, X_i]
    = \E[Y_{i,\ell}| G_i, X_i]
    - \E[\E[Y_{i,\ell}|X_i]|G_i, X_i]
    = \E[Y_{i,\ell}| X_i] - \E[Y_{i,\ell}|X_i] = 0,\label{eq:17}
  \end{equation}
  where the first term has been simplified using the conditional independence of
  $Y_i$ from $G_i$ given $X_i$, and the second term using the tower property of
  conditional expectation.  We thus get
  \begin{equation}
    \label{eq:18}
    \E[T_{i,\ell}\hV_{i,\ell}| G_i]
    = \E[\E[T_{i,\ell}\hV_{i,\ell}|X_{i}, G_i]| G_i]
    = \E[\hV_{i,\ell}\E[T_{i,\ell}|X_{i}, G_i]| G_i]
    \overset{\text{\tiny\cref{eq:17}}}{=} 0,
  \end{equation}
  where the first equality is the tower property, and the second uses the
  observation from above that $\hV_{i,\ell}$ is measurable with respect to the
  $\sigma$-field generated by the random variables $G_i$ and $X_i$.  From the
  independence of the samples, we also have that the above equalities hold when
  conditioning on $G$:
  \begin{equation}
    \E[T_{i,\ell}\hV_{i,\ell}|G] = \E[T_{i,\ell}\hV_{i,\ell}] = 0.\label{eq:19}
  \end{equation}
  Now, from \cref{item:phi-grad-lip} in \cref{maintheorem}, we see that the
  function $(\nabla\phi)_{\ell} \circ f: \R^d \rightarrow \R$ is
  $\phidL \cdot L$-lipschitz.  From the $c$-isoperimetric assumption on each
  mixture component of the co-variate distribution, it therefore follows that
  \emph{conditioned on $G$}, for each $i \in [n], \ell \in [K]$, the
  (conditionally) mean-zero random variable $\hV_{i, \ell}$ is sub-gaussian with
  parameter $\phidL{}L\sqrt{c/d}$.  Combining \cref{eq:19} with the absolute
  bound of $\omegadiam$ on $T_{i,\ell}$ given above and with
  \cref{fct-multiply-sub-gauss}, it thus follows that \emph{conditioned on $G$},
  the random variables $T_{i,\ell}\hV_{i,\ell}$ are mean-zero and sub-gaussian
  with parameter $C\omegadiam\phidL{}L\sqrt{c/d}$.  Further, by the independence
  of samples, it also follows that for any fixed $\ell$, (even when conditioned
  on $G$) they are independent.  Thus, the Hoeffding inequality gives that for
  any $1 \leq \ell \leq k$,
  \begin{equation}
    \Pr\insq{\frac{1}{n}\sum_{i=1}^nT_{i,\ell}\hV_{i,\ell} \leq -\epsilon \bigg| G}
    \leq \exp\inp{-\frac{nd\epsilon^2}{2cC^2\omegadiam^2L^2\phidL^2}}.
  \end{equation}
We take expectations on both sides to get the claimed
  ``un-conditioned'' bound.
 \end{proof}

  \begin{lemma}
    \label{lem-vt-conc}
    With the notation above, we have for every $1 \leq \ell \leq K$ and every
    $\epsilon > 0$
   \begin{equation}
     \Pr\insq{\inf\limits_{f \in \F} \frac{1}{n}\sum_{i=1}^nT_{i,\ell}\tV(f)_{i,\ell} \leq -\epsilon}
     \leq  2\compnum\exp\inp{\frac{-n\epsilon^2}{8\gamma^{2}\compnum\omegadiam^2}}.\label{eq:24}
   \end{equation}
 \end{lemma}

 \begin{proof}
   Note that for every $f \in \F$ the random variables $\tV(f)_{i,\ell}$ are
   $G_i$-measurable and are bounded as $\abs{\tV(f)_{i,\ell}} \leq 2\gamma$ (by
   \cref{eq:global-bounds}).  We thus see that for any fixed $\ell$,
   \emph{conditioned on the mixture labels $G$},
   \begin{enumerate}
   \item the random variables
     $\tV(f)_{i,\ell}$ become deterministic with absolute value at most
     $2\gamma$. In fact, $\tV(f)_{i, \ell} = \tV(f)_{{j, \ell}}$ whenever
     $G_i = G_j$.\label{item:1}
   \item the $T_{i,\ell}$ are \emph{independent} (though not identically
     distributed), have absolute value at most $\omegadiam$, and have
     (conditional) mean $0$ (as argued above in \cref{eq:16}, and the paragraph
     preceding it).\label{item:2}
   \end{enumerate}
   Following Bubeck and Sellke~\cite{BSisoperimetry}, we thus define the
   (random) sets $S_k \defeq \inb{i \in [n] | G_i = k}$, for
   $1 \leq k \leq \compnum$ (note that the random sets $S_k$ are deterministic
   conditioned on $G$), and then use the Cauchy-Schwarz inequality to obtain
   \begin{equation}
     \label{eq:s-size}
     \sum_{k=1}^\compnum\sqrt{\abs{S_k}} \leq
     \sqrt{r \cdot \sum_{k=1}^{\compnum}\abs{S_k}} \leq \sqrt{n\compnum}.
   \end{equation}
   From
   \cref{item:1} above, $\tV(f)_{i,\ell}$ is the same for each $i \in S_k$, and
   is at most $2\gamma$ in absolute value.  Thus, for each
   $1 \leq k \leq \compnum$ we have
   \begin{equation}
     \Pr\insq{
       \inf\limits_{f \in \F}
       \sum_{i\in S_{k}}T_{i,\ell}\tV(f)_{i,\ell}
       \leq -\epsilon\sqrt{\frac{n\abs{S_k}}{\compnum}}
       \bigg| G
     }
     \leq
     \Pr\insq{
       \abs{\sum_{i\in S_{k}}T_{i,\ell}}
       \geq
       \frac{\epsilon}{2\gamma}
       \sqrt{\frac{n\abs{S_k}}{\compnum}}
       \bigg| G
     }.
   \end{equation}
   From \cref{item:2} above, the Hoeffding inequality can be applied to the
   right hand side above, so that we get
   \begin{equation}
     \Pr\insq{
       \inf\limits_{f \in \F}
       \sum_{i\in S_{k}}T_{i,\ell}\tV(f)_{i,\ell}
       \leq -\epsilon\sqrt{\frac{n\abs{S_k}}{\compnum}}
       \bigg| G
     }
     \leq
     2\exp\inp{\frac{-n\epsilon^2}{8\gamma^{2}\compnum\omegadiam^2}}.\label{eq:25}
   \end{equation}
   Finally, by a union bound we have
   \begin{align*}
     \Pr\insq{
     \inf\limits_{f \in \F} \sum_{i=1}^nT_{i,\ell}\tV(f)_{i,\ell} \leq -n\epsilon
     \bigg|G
     }
     &
       \leq
       \sum_{k=1}^{\compnum}
       \Pr\insq{
       \inf\limits_{f \in \F} \sum_{i\in S_{k}} T_{i,\ell}\tV(f)_{i,\ell}
       \leq \frac{-n\epsilon\sqrt{\abs{S_k}}}{\sum_{t=1}^\compnum\sqrt{\abs{S_t}}}
       \bigg|G
       }\\
     &
       \overset{\text{\tiny\cref{eq:s-size}}}{\leq}
       \sum_{k=1}^{\compnum}
       \Pr\insq{
       \inf\limits_{f \in \F}
       \sum_{i\in S_{k}}T_{i,\ell}\tV(f)_{i,\ell}
       \leq -\epsilon\sqrt{\frac{n\abs{S_k}}{\compnum}}
       \bigg| G
       }\\
     &\overset{\text{\tiny\cref{eq:25}}}{\leq}
       2\compnum\exp\inp{\frac{-n\epsilon^2}{8\gamma^{2}\compnum\omegadiam^2}}.
   \end{align*}
   The claim now follows by taking expectations of both sides in the above.
 \end{proof}

 We can now describe the modifications needed to complete the proof of
 \cref{maintheorem}.  These modifications are described in terms of the notation
 set up above and in the proof for the case $r = 1$ given in
 \cref{sec:proof-main-theorem}.  In particular, the $\nu$-net $\F_{L,\nu}$ of
 the subset $\F_L$ of $L$-lipschitz functions in $\F$ is as defined in that
 proof.
\begin{proof}[Proof of \cref{maintheorem}: Case $r > 1$] Most of the proof of
  the theorem remains the same, except for use of \cref{lem-Gamma-3} in the
  derivation of \cref{eq:5}, which now has to be replaced by applications of
  \cref{lem-vhat-conc,lem-vt-conc}.  First, using the decompositions in
  \cref{eq:26,eq:20} above and a union bound, we have
  \begin{multline}
    \Pr\insq{\exists f \in \F_{L,\nu} \text{ s.t. }
      \frac{1}{n}\sum_{i=1}^n\Gamma_3^{(i)}(f) \leq - \epsilon/8
    } \\ \leq
    \sum_{\ell=1}^K\inp{
      \Pr\insq{\inf\limits_{f \in \F_{L,\nu}} \frac{1}{n}\sum_{i=1}^nT_{i,\ell}\tV(f)_{i,\ell} \leq \frac{-\epsilon}{16K}}
      +
      \Pr\insq{\inf\limits_{f \in \F_{L,\nu}} \frac{1}{n}\sum_{i=1}^nT_{i,\ell}\hV(f)_{i,\ell} \leq \frac{-\epsilon}{16K}}
    }. \label{eq:27}
  \end{multline}
   By construction,
  $\F_{L,\nu} \subseteq \F$, so we bound the first term using
  \cref{lem-vt-conc}.
  \begin{equation}
    \label{eq:28}
    \begin{aligned}
      \sum_{\ell=1}^K
      \Pr\insq{\inf\limits_{f \in \F_{L,\nu}} \frac{1}{n}\sum_{i=1}^nT_{i,\ell}\tV(f)_{i,\ell} \leq \frac{-\epsilon}{16K}}
      &\leq
      \sum_{\ell=1}^K
      \Pr\insq{\inf\limits_{f \in \F}
      \frac{1}{n}\sum_{i=1}^nT_{i,\ell}\tV(f)_{i,\ell}
      \leq \frac{-\epsilon}{16K}}\\
      &  \leq  2K\compnum\exp\inp{\frac{-n(\epsilon/16)^2}{8K^2\gamma^{2}\compnum\omegadiam^2}}.
    \end{aligned}
  \end{equation}
  For the second term, we use \cref{lem-vhat-conc} and a union bound over the
  $\nu$-net $\F_{L, \nu}$ of the set of $L$-lipschitz function in $\F$, exactly
  as in the argument leading to \cref{eq:5}.  This gives,
  \begin{equation}
    \sum_{i=1}^{K}
    \Pr\insq{\inf\limits_{f \in \F_{L,\nu}} \frac{1}{n}\sum_{i=1}^nT_{i,\ell}\hV(f)_{i,\ell} \leq \frac{-\epsilon}{16K}}
    \leq K \abs{\F_{L,\nu}}
    \cdot \exp\inp{-\frac{nd(\epsilon/16)^2}{2cC^2K^2\omegadiam^2L^2\phidL^2}}.
  \end{equation}
  Together, these two computations, when substituted into \cref{eq:27}, give the
  following more general version of \cref{eq:5}.
  \begin{multline}
    \label{eq:29}
    \Pr\insq{\exists f \in \F_{L,\nu} \text{ s.t. }
      \frac{1}{n}\sum_{i=1}^n\Gamma_3^{(i)}(f) \leq - \epsilon/8
    } \leq
    2K\compnum\exp\inp{\frac{-n(\epsilon/16)^2}{8K^2\gamma^{2}\compnum\omegadiam^2}}
    +
    K \abs{\F_{L,\nu}}
    \cdot \exp\inp{-\frac{nd(\epsilon/16)^2}{2cC^2K^2\omegadiam^2L^2\phidL^2}}.
  \end{multline}
  We then proceed with the argument exactly as before, replacing all usages of
  \cref{eq:5} by \cref{eq:29}.  The final bound in \cref{eq:30} then gets
  modified to the following.
  \begin{equation}
    \Pr\insq{\exists f \in \F_L \text{ s.t. }
      \frac{1}{n}\sum_{i=1}^nZ_i(f) - \sigma_{\phi}^2 \leq - \epsilon
    }
    \leq  K \abs{\F_{L,\nu}}
    \cdot e^{-\frac{nd(\epsilon/16)^2}{2cC^2K^2\omegadiam^2L^2\phidL^2}}
    +     2K\compnum\exp\inp{\frac{-n(\epsilon/16)^2}{8K^2\gamma^{2}\compnum\omegadiam^2}}
    +  2K\sum_{j=0}^2 e^{-2n(\epsilon/8)^{2}/(M_j)^{2}}.
  \end{equation}
  As in the proof for the $\compnum = 1$ case, each term above is at most
  $\delta/5$ by the choice of the parameters.
\end{proof}

\section{Discussion}

In this paper, we gave a more comprehensive understanding of the law of robustness of Bubeck and Sellke \cite{BSisoperimetry} for interpolation by considering Bregman divergence losses. In applications, the objective of interest is usually robust generalization rather than robust interpolation. We leave the extension of the line of work of this paper to robust generalization as open. We also suggest a few specific directions of inquiry:
\begin{enumerate} 
\item To understand robust generalization in practice, \emph{local} notions of the Lipschitz constant are often tighter than the global notion \cite{huang2021training,muthukumar2023adversarial}. Bubeck and Sellke \cite{BSisoperimetry} remark that the expected squared norm of the gradient instead of the (global) Lipschitz constant does not lead to a similar law of robustness. How can we get a better understanding of robustness for local notions of Lipschitz constants? 
\item The current line of work focuses on models that overfit training data, as such models heralded the recent deep learning revolution. However, modern model training procedures often do not fall into this paradigm. In fact they may practice `early stopping', for which a theory of overparameterization has recently been proposed \cite{li2020gradient}. Connecting the line of work of robust interpolation to practical `underfitting' setups such as early stopping is an important research problem.
\end{enumerate}

\appendix

 \section{Proofs of corollaries of the main theorem}
\label{sec:proofs-coroll-main}
\begin{proof}[Proof of \cref{cor-regression}]
  The square loss is a symmetric Bregman divergence loss $D_{\phi}$ with
  $\phi(x)=\norm{x}^{2}$.  We consider the domain $\Omega = [-M, M]^K$ for
  $\phi$, which has $l_{\infty}$ diameter at most $\omegadiam=2M$. Further, on
  this domain, $\phi$ is a $2\sqrt{K}M$-Lipschitz function, so we can take
  $\phiL=2\sqrt{K}M$.  Since $\nabla \phi(x) = 2x$, we can take $\phidL = 2$.

  The set $A=\{\E[Y|x]:x\in \R^{d}\}$ is (by definition) contained in
  $[-M,M]^{K}$. We thus get
  $$a_{0}=\max_{a\in A} \norm{a} \leq m_{0}=\max_{a\in [-M, M]^{K}} \norm{a} =
  \sqrt{K}M,$$
  $$m_{2}=\max_{a\in A}\abs{\phi(a)} \leq m_{1}=\max_{a\in [-M, M]^{K}}
  \abs{\phi(a)} = KM^{2},$$
$$m_{3}=\max_{a\in A}||\nabla \phi(a)||\leq \max_{a\in [-M, M]^{K}}||\nabla \phi(a)||= 2\sqrt{K}M,\hspace{2pt} \gamma\leq 2\sqrt{K}M.$$
We have thus verified all assumptions of \cref{maintheorem}, and using that
Theorem, we now get that for any $(p,J)$-realistic function class $\F$,
w.h.p. over the sampling of $n$ independent samples (where $n$ is as in the
statement of the corollary), if there exists an $L$-Lipschitz function
$f \in \F$ that $\epsilon$-overfits the training data, then the following lower
bound on $L$ holds:
\begin{displaymath}
  L\geq \frac{\epsilon}{128CKM\sqrt{2c}}
  \sqrt{
    \frac{
      nd
    }{
      p\log (1+64JWKM)/\epsilon)
      + \log(5K/\delta)
    }
  }. \qedhere
\end{displaymath}
\end{proof}

\begin{proof}[Proof of \cref{cor-classification}]
  Let $\tilde{\Delta}_K$ denote the $K$ dimensional probability simplex.  The
  cross-entropy loss is an asymmetric Bregman divergence loss $D_{\phi}$ with
  $\phi:\tilde{\Delta}_{K}\rightarrow \mathbb{R}$ defined as
  $\phi(x) \defeq \sum_{i=1}^{K}x_{i}\log(x_{i})$.  Note also that the
  $ \ell_{\infty}$-diameter of $\Omega = \tilde{\Delta}_{K}$ is at most
  $\omegadiam=1$.  We also recall that the labels $Y$ are one-hot encoded
  vectors, and the assumption that $\Pr[Y_i = 1|x] \geq \alpha$ for all
  $i \in [K]$ and $x$. This implies that
  $A=\{\E[Y|x]=(\Pr[Y_1=1 |x],\Pr[Y_2=1|x],\ldots,\Pr[Y_K=1|x]) : x\in
  \R^{d}\}\subseteq \tilde{\Delta}_{K} \cap [\alpha, 1-\alpha)^K$.

  Note also that the range of any function in the $(p, J)$-realistic function
  class $\F$ is a subset of
  $B \defeq \{\operatorname{Softmax}(x):x\in [-M,M]^{K}\}\subset B_1 \defeq \{x\in
  \tilde{\Delta}_{K}:x_i\geq \frac{e^{-2M}}{K} \hspace{2pt}\forall 1\leq i\leq
  K\}$, and $\nabla\phi(x) = (\log(x_1)+1,\log(x_2)+1,\ldots,\log(x_K)+1)$ is
  well defined for all $x \in B_1$.  From this it follows that $\phi$ is an
  $\phiL$-Lipschitz function on $B_1$ with $\phiL=\sqrt{K}(1 + 2M + \log K)$.
  We also have
  $$a_{0}=\max_{a\in A} \norm{a} \leq m_{0}=\max_{b\in \tilde{\Delta}_{K}} \norm{b} = 1,
  $$ $$m_{2}=\max_{a\in A} \abs{\phi(a)} \leq m_{1}=\max_{b\in
    \tilde{\Delta}_{K}}  \abs{\phi(b)} \leq \log K, $$
  $$m_{3}=\max_{a\in A} \norm{\nabla \phi(a)}
  \leq \sqrt{K}(1+|\log(\alpha)|), \gamma\leq \max_{b\in B_1}\norm{\nabla
    \phi(b)} = \sqrt{K}(1 + 2M + \log K).$$

  Further, we also see that each coordinate $(\nabla\phi)_{\ell}$ of the
  gradient of $\phi$ is an $\phidL$-Lipschitz function on the set $B_1$ with
  $\phidL =Ke^{2M}$.  Thus, we have verified all assumptions of
  \cref{maintheorem}, and using that theorem, we get that for any
  $(p,J)$-realistic function class $\F$, w.h.p. over the sampling of $n$
  independent samples (where $n$ is as in the statement of the corollary), if
  there exists an $L$-Lipschitz function $f \in \F$ that $\epsilon$-overfits
  then the following lower bound on $L$ holds:
\begin{equation}
    L\geq \frac{\epsilon}{32CK^{2}e^{2M}\sqrt{2c}}
  \sqrt{
    \frac{
      nd
    }{
      p\log (1+8JW( e^{2M}K^{2} + 2\sqrt{K}(1 + 2M +  \log K))/\epsilon)
      + \log(5K/\delta)
    }
  }.
\end{equation}
This bound can however be improved if we take a careful look at the internals of
the proof of the \cref{maintheorem}.  The crucial observation is that in
\cref{lem-Gamma-3,lem-vhat-conc}, we only need the Lipschitz constant of the
composition $(\nabla \phi)_{\ell}\circ f$, which we estimate in general by
multiplying the worst-case Lipschitz constants of $\nabla \phi_{\ell}$ and
$f \in \F$.  However, in this case, for $\sf(f) \in \F$, we can get a better
direct estimate of this Lipschitz constant of $(\nabla \phi)_{\ell}\circ \sf(f)$ since
$(\nabla \phi)_{\ell}\circ \sf(f) =
\log(e^{f(x)_{\ell}}/\sum_{t=1}^{K}e^{f(x)_{t}}) + 1$ is a $2L$-Lipschitz
function from $\R^{d}$ to $\R^{K}$ whenever $f$ is an $L$-Lipschitz function. Thus, we can replace the
factor $L\phidL$ appearing in the proof of \cref{maintheorem} by $2L$.  This
propagates through the proof of \cref{maintheorem} and we thus obtain the
following w.h.p.~lower bound on the Lipschitz constant $L$ of $g$ when $\sf(g)
\in \F$ over-fits the data:
\begin{equation}
  L\geq \frac{\epsilon}{64CK\sqrt{2c}}
  \sqrt{
    \frac{
      nd
    }{
      p\log (1+8JW( e^{2M}K^{2} + {\sqrt{K}}(1 + 2M + \log K))/\epsilon)
      + \log(5K/\delta)
    }
  }.\qedhere
\end{equation}

\end{proof}

\end{document}